\documentclass[nohyperref]{article}

\usepackage{microtype}
\usepackage{graphicx}
\usepackage{booktabs} 

\usepackage{hyperref}

\usepackage[accepted]{icml2023}

\usepackage{amsmath}
\usepackage{amssymb}
\usepackage{mathtools}
\usepackage{amsthm}

\usepackage[capitalize,noabbrev]{cleveref}

\theoremstyle{plain}
\newtheorem{theorem}{Theorem}[section]
\newtheorem{proposition}[theorem]{Proposition}
\newtheorem{lemma}[theorem]{Lemma}

\theoremstyle{definition}

\theoremstyle{remark}

\usepackage[textsize=tiny]{todonotes}

\usepackage[noend,ruled,linesnumbered,algo2e]{algorithm2e}
\DontPrintSemicolon 
\SetKwInput{KwInput}{Input}
\SetKwInput{KwOutput}{Output}
\usepackage{subcaption}
\usepackage{enumitem}
\usepackage{mathtools}
\usepackage{tikz}
\usepackage{amsmath}
\usepackage{amssymb}
\usepackage{amsthm}
\usetikzlibrary{positioning, fit, backgrounds, calc, decorations, decorations.markings, shapes, shapes.misc, arrows.meta, shadows, hobby}
\usepackage{multirow}
\usepackage{standalone}
\usepackage{import}
\usepackage{bbm} 
\usepackage{bbold}
\usepackage[nice]{nicefrac}

\usepackage{pgfplots}
\pgfplotsset{compat=newest}

\usepackage{pifont}
\DeclareMathOperator*{\argtopk}{\arg \textnormal{top-}k}

\DeclareMathOperator*{\R}{\mathbb{R}}

\DeclarePairedDelimiter\abs{\lvert}{\rvert}%
\DeclareMathOperator*{\argmax}{arg\,max}

\providecommand{\la}{\langle}
\providecommand{\ra}{\rangle}

\providecommand{\abs}[1]{\lvert #1 \rvert}

\providecommand{\A}{\mathcal{A}}

\providecommand{\alg}[1]{\normalfont{\texttt{#1}}}

\providecommand{\N}{\mathcal{N}}

\newcommand{\1}{\mathbb{1}}

\definecolor{scattercolor1}{HTML}{800080}
\definecolor{scattercolor2}{HTML}{990000}
\definecolor{scattercolor3}{HTML}{009900}
\definecolor{scattercolor4}{HTML}{808000}
\definecolor{scattercolor5}{HTML}{D30E78}
\definecolor{scattercolor6}{HTML}{036ffc}
\definecolor{scattercolor7}{HTML}{eb9800}

\icmltitlerunning{Clustering Fully connected Graphs by Multicut}

\begin{document}
\twocolumn[
\icmltitle{ClusterFuG: Clustering Fully connected Graphs by Multicut}



\icmlsetsymbol{equal}{*}

\begin{icmlauthorlist}
\icmlauthor{Ahmed Abbas}{mpi}
\icmlauthor{Paul Swoboda}{mpi,uni_mann}
\end{icmlauthorlist}

\icmlaffiliation{mpi}{MPI for Informatics, Saarland Informatics Campus, Germany}
\icmlaffiliation{uni_mann}{University of Mannheim, Germany}

\icmlcorrespondingauthor{Ahmed Abbas}{ahmed.abbas@mpi-inf.mpg.de}

\icmlkeywords{clustering, multicut, combinatorial optimization}

\vskip 0.3in
]


\printAffiliationsAndNotice{}  

\begin{abstract}
We propose a graph clustering formulation based on multicut (a.k.a.\ weighted correlation clustering) on the complete graph.
Our formulation does not need specification of the graph topology as in the original sparse formulation of multicut, making our approach simpler and potentially better performing.
In contrast to unweighted correlation clustering we allow for a more expressive weighted cost structure.
In dense multicut, the clustering objective is given in a factorized form as inner products of node feature vectors.
This allows for an efficient formulation and inference in contrast to multicut/weighted correlation clustering, which has at least quadratic representation and computation complexity when working on the complete graph.
We show how to rewrite classical greedy algorithms for multicut in our dense setting and how to modify them for greater efficiency and solution quality.
In particular, our algorithms scale to graphs with tens of thousands of nodes.
Empirical evidence on instance segmentation on Cityscapes and clustering of ImageNet datasets shows the merits of our approach.
\end{abstract}
\section{Introduction}
\label{sec:introduction}
Graph-based clustering approaches, primarily among them multicut~\cite{chopra1993partition}, are theoretically appealing: 
They do not need specification of the number of clusters, but infer them as part of the optimization process.
They allow for a flexible clustering objective with attractive and repulsive costs between pairs of nodes.
They are also theoretically well-understood as optimization problems with intensively studied polyhedral descriptions. 
Efficient solvers that scale well and give high quality solutions have also been developed.

As a drawback, graph-based clustering approaches need specification of the underlying graph topology.
In practice, this means an additional engineering effort as well as the possibility to not get it right, which would decrease the downstream task performance.
Naively circumventing this challenge by using the complete graph is not scalable -- the number of edges grows quadratically.
One approach to resolve this conundrum is graph structure learning e.g., by extending~\cite{kazi2022differentiable}, but adds considerable additional complexity.

We propose a method to solve graph clustering efficiently on complete graphs.
Our formulation will use the well-known edge-based multicut formulation and only restrict the way edge costs can be computed: they need to be based on inner products of node features.
This has two advantages:
First, it reduces storage requirements.
Instead of storing a full adjacency matrix of edge costs as in multicut, which grows quadratically with the number of nodes, we only need to store a linear number of node features and can compute edge costs on demand.
Second, operations needed in multicut algorithms can be made scalable.
Instead of operating on the complete graph we can sparsify it adaptively during the solving process.
This allows to simulate the workings of multicut algorithms on complete graphs by working on a small subset of it.
The key technical ingredient to obtain these sparse subgraphs will be fast nearest neighbor search, for which efficient and scalable implementations exist~\cite{johnson2019_faiss}.
In effect, this allows us to solve large dense multicut instances in moderate time, which is not possible with existing solvers. In detail, our contribution is as follows:
\begin{description}
\item[Formulation:] We propose multicut on complete graphs with factorized edge costs as an efficiently representable graph clustering formalism.
\item[Algorithm:] We propose scalable algorithms for solving the dense multicut problems, one mimicking exactly the original greedy additive edge constraction (GAEC) algorithm~\cite{keuper2015efficient}, the other a more efficient variant in the spirit of the balanced edge contraction heuristic~\cite{kardoost2018solving}.
\item[Empirical:] We show efficacy in terms of memory and runtime of our solvers and show the merit of using them for image segmentation on Cityscapes and clustering of ImageNet classification dataset. Our implementation is available at \url{https://github.com/aabbas90/cluster-fug}.
\end{description}

\section{Related work}
\label{sec:related-work}

\paragraph{Multicut and correlation clustering:}
The original multicut problem is formulated as an extension of the min-cut problem to multiple terminals with non-negative edge costs~\cite{hu1963multicut}.
In machine learning the multicut problem is defined differently and is equivalent (up to variable involution) to the correlation clustering problem~\cite{demaine2006correlation}, i.e.\ arbitrary edges costs and no terminals.
For the purpose of this work we will use the latter definition of multicut.
The polyhedral geometry of the multicut problem has been studied in~\cite{deza1992clique,chopra1993partition,oosten2001clique}. An equivalent problem for multicut on complete graphs is the clique partitioning problem studied in~\cite{grotschel1990facets}.

Although the multicut problem is NP-Hard~\cite{bansal2004correlation, demaine2006correlation}, greedy algorithms perform well in practice for computer vision and machine learning tasks~\cite{keuper2015efficient,levinkov2017comparative,bailoni2022gasp}.
More involved algorithms include message passing in the dual domain for multicut, studied in~\cite{swoboda2017message, lange2018partial, rama_22}.
These algorithms give lower bounds and improved primal solutions.
Another line of efficient primal heuristics is based on move-making~\cite{beier2014cut,beier2015fusion}.
All these graphs, while efficient, scale with the number of edges, making them unsuitable for very large dense graphs.
Algorithms for correlation clustering on complete graphs were proposed in~\cite{pan15corr_cluster_big_graphs, veldt2022correlation_clustering}.
However, they only allow unweighted edges.
In this paper we consider efficient algorithms on full graphs and with weighted edges.

\paragraph{$K$-Means:}
The $K$-means problem~\cite{lloyd_kmeans} is similar to our approach in that it works directly on feature representations and its objective is based on $L_2$-distances between features.
Similarly to our algorithm, large number of points are handled by efficiently computing kNN-graphs~\cite{qaddoura2020efficientkmeans_nn}, thereby reducing run time.
In contrast to multicut, the number of clusters must be given a-priori, while in multicut it is derived as part of the optimization process.

\paragraph{Other clustering approaches:}
There are a number of other paradigms for clustering.
A prominent approach is spectral clustering, in which a weighted graph is given and a clustering is computed with the help of the eigenvectors of the graph Laplacian~\cite{von07tutorial_spectral_clustering, jia14spectral_clustering_survey}.
The work~\cite{dhillon07graculus} shows connections between weighted $k$-means and multiple spectral clustering approaches.
As for K-means and unlike multicut, spectral clustering requires the number of clusters to be specified.

\section{Method}
\label{sec:method}

A \emph{decomposition (or clustering)} of a weighted graph $G = (V, E, c)$ with vertices $V$, edges $E$ and edge costs $c \in \R^{E}$ can be obtained by solving the following multicut problem 
\begin{equation}
  \label{eq:multicut}
  \min_{y \in \mathcal{M}_G} \sum_{ij \in E} c_{ij}y_{ij}.
\end{equation}
We say that an edge $ij$ with $c_{ij} > 0$ is \emph{attractive}. Its endpoints prefer to be in the same cluster.
In the opposite case $c_{ij} < 0$ we call the edge \emph{repulsive}.
The set $\mathcal{M}_G$ enumerates all possible partitions of $G$ defined as 
\begin{equation}
  \mathcal{M}_G = \left\{ \1_{\delta(V_1,\ldots,V_n)} : 
\begin{array}{c}
  n \in \mathbb{N} \\
  V_i \cap V_j = \varnothing \quad \forall i\neq j\\
  V_1 \dot\cup \ldots \dot\cup V_n = V
\end{array} \right\}\,.
\end{equation} 
where $\delta(\cdot,\ldots,\cdot) \subseteq E$ is the set of edges straddling distinct components and $\1_\delta$ is the indicator vector of $\delta$. 
\begin{figure}
    \centering
    \scalebox{1.5}{\def\connect#1#2#3{%
    \path let
      \p1 = ($(#2)-(#1)$),
      \n1 = {veclen(\p1)},
      \n2 = {atan2(\y1,\x1)} 
    in
      (#1) -- (#2) node[#3, midway, rotate = \n2, shading angle = \n2+90, minimum        width=\n1, minimum height=1pt, inner sep=1pt] {};
}
\begin{tikzpicture}
\definecolor{scattercolor1}{HTML}{800080}
\definecolor{scattercolor2}{HTML}{DD0000}
\definecolor{scattercolor3}{HTML}{009900}
\definecolor{scattercolor4}{HTML}{808000}
\definecolor{scattercolor5}{HTML}{D30E78}
\definecolor{scattercolor6}{HTML}{036ffc}
\definecolor{scattercolor7}{HTML}{eb9800}
\tikzstyle{cut-edge}=[gray!70, dashed]
\tikzstyle{vector-arrow}=[gray!70, densely dotted]
\tikzstyle{att-edge}=[scattercolor3!70]
\tikzstyle{rep-edge}=[red!70]
\tikzstyle{vertex}=[circle, draw, ultra thin, fill=white, inner sep=0.7pt, minimum width=9pt]
\tikzstyle{vector}=[rectangle, draw,inner sep=1pt,outer sep=0pt,
minimum height=5mm, minimum width=2mm, ultra thin]

\node[vertex] (Pi) at ({90}:2.0cm) {\small{$i$}};
\node[vertex] (Pj) at ({44}:1.6cm) {\small{$j$}};
\node[vertex] (Pk) at ({290}:1.cm) {};
\node[vertex] (Pl) at ({250}:1.cm) {};
\node[vertex] (Pm) at ({136}:1.6cm) {};

\draw[->, style=vector-arrow] (0,0) -- (Pi) node [pos=0.3,fill=white,scale=0.8] {{$f_i$}};
\draw[->, style=vector-arrow] (0,0) -- (Pj) node [pos=0.4,fill=white,scale=0.8] {{$f_j$}};

\draw[fill=gray, gray] (0,0) circle (.2ex);
\node[draw=none, gray, scale=0.4] (origin) at (-0.0, -0.1) {$(0, 0)$};
\draw[style=att-edge, line width = 1.4pt] (Pi) -- (Pj);
\draw[style=rep-edge, line width = 1.1pt, dashed] (Pi) to [bend left=30] (Pk);
\draw[style=rep-edge, line width = 1.1pt, dashed] (Pi) to [bend right=30] (Pl);
\draw[style=att-edge, line width = 1.4pt] (Pi) -- (Pm);
\draw[style=rep-edge, line width = 0.6pt, dashed] (Pj) to [bend left=30] (Pk);
\draw[style=rep-edge, line width = 0.95pt, dashed] (Pj) to [bend left=30] (Pl); 
\draw[style=rep-edge, line width = 0.2pt] (Pj) -- (Pm); 
\draw[style=att-edge, line width = 0.6pt] (Pk) -- (Pl); 
\draw[style=rep-edge, line width = 0.95pt, dashed] (Pk) to [bend left=30] (Pm); 
\draw[style=rep-edge, line width = 0.6pt, dashed] (Pl) to [bend left=30] (Pm); 

\begin{pgfonlayer}{background}

\draw[fill=blue!50,opacity=0.2,draw=blue]([yshift=0.1cm] Pi.north) to[closed,curve through={([yshift=0.1cm]Pi.north east) .. ([yshift=0.1cm]Pj.north) .. ([xshift=0.1cm]Pj.east) .. ([yshift=-0.1cm]Pj.south) .. ([yshift=-0.15cm]Pj.south west) .. ([yshift=-0.15cm]Pm.south east) .. ([yshift=-0.1cm]Pm.south) .. ([xshift=-0.1cm]Pm.west) .. ([yshift=0.1cm]Pm.north) .. ([yshift=0.1cm]Pi.north west)}]([yshift=0.1cm] Pi.north);

\draw[fill=blue!50,opacity=0.2,draw=blue]([yshift=0.1cm] Pk.north) to[closed,curve through={([yshift=0.1cm]Pl.north) .. ([xshift=-0.1cm]Pl.west) ([yshift=-0.1cm] Pl.south) .. ([yshift=-0.1cm]Pk.south)}]([xshift=0.1cm] Pk.east);

\end{pgfonlayer}

\end{tikzpicture}}
    \caption{Example illustration of dense multicut problem~\eqref{eq:dense_multicut} on $5$ nodes. Each node $i$ is associated with a vector $f_i \in \R^2$ and all possible edges between distinct nodes are considered (i.e., the complete graph). The edge cost between a pair of nodes $i$, $j$ is measured by $\langle f_i, f_j \rangle$ and attractive/repulsive edges are colored \textcolor{green}{green}/\textcolor{red}{red}. Edge thickness represents absolute edge cost. Also shown is the optimal partitioning to $2$ clusters with cut edges denoted by dashed lines.}
    \label{fig:complete-multicut-illustration}
\end{figure}
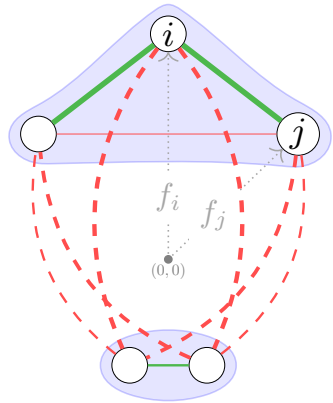

The goal of our work is to consider the scenario when the graph $G$ is complete i.e., $E = \{ij : i \in V, j \in V\setminus \{i\} \}$. For large graphs storage and processing of edge costs $c$ becomes prohibitive. To address this issue we instead require as input a feature vector $f_i \in \R^d$ for each node $i$ in $V$. The edge costs between a pair of nodes $i$ and $j$ can then be measured on-demand through some function $s(f_i, f_j) \rightarrow \R$. In this case the multicut problem becomes
\begin{equation}
  \label{eq:dense_multicut}
  \min_{y \in \mathcal{M}_G} \sum_{i \in V} \sum_{j \in V \setminus i} s(f_i, f_j) y_{ij},
\end{equation}
which we term as dense multicut problem. An illustration of our formulation is given in Figure~\ref{fig:complete-multicut-illustration}. 
In the following we first revisit an algorithm to approximately solve~\eqref{eq:multicut} and show its extensions for dense multicut problem~\eqref{eq:dense_multicut}. 

\subsection{Greedy Additive Edge Contraction}
The greedy additive edge contraction (GAEC) scheme~\cite{keuper2015efficient} computes approximate solution of the multicut problem~\eqref{eq:multicut} as given in Algorithm~\ref{alg:gaec}. It initializes each node as a separate cluster and iteratively contracts a pair of nodes $i$, $j$ with the largest non-negative cost $c_{ij}$ (if it exists).
Let $m$ be the node $i$ and $j$ are contracted to.
The edge costs of edges incident to $m$ are
\begin{equation}
\label{eq:contraction_gaec}
c_{ml} = c_{il} + c_{jl}, l \in \N_i \cup \N_j \setminus \{i, j\},
\end{equation}
where costs of non-existing edges are assumed to be $0$ and $\N_i$ corresponds to neighbours of $i$ in graph $G$.
For complete graphs directly applying this algorithm by operating on edge costs is computationally expensive. 
Moreover, since each node is connected to all other nodes ($\N_i = V \setminus \{i\}$), cost updates~\eqref{eq:contraction_gaec} during edge contraction take $\mathcal{O}(\abs{V})$ instructions.

\begin{algorithm2e}
\newcommand\mycommfont[1]{\footnotesize\ttfamily\textcolor{blue}{#1}}
\SetCommentSty{mycommfont}
\caption{\texttt{GAEC~\cite{keuper2015efficient}}}
\label{alg:gaec}
\KwData{Weighted graph $G = (V, E, c)$}
\KwResult{
Clusters $V$
}
\While{$\max_{uv \in E} c_{uv} \geq 0$}
{
$m \coloneqq ij = \argmax_{uv \in E} c_{uv}$\;
\tcp{Aggregate edge costs}
$c_{ml} = c_{il} + c_{jl},\, l \in \N_i \cup \N_j \setminus \{i, j\}$
\;\label{alg:aggregate_edge_costs}
\tcp{Update edges}
$ E' = \{ml \vert l \in \N_i \cup \N_j \setminus \{i, j\}\}$\; 
$ E = E' \cup E \setminus \{il\}_{l \in \N_i} \cup \{jl\}_{l \in \N_j} $\;
\tcp{Update nodes}
$ V = (V \cup \{m\}) \setminus \{i, j\}$ \; 
}
\end{algorithm2e}

\paragraph{Contraction on complete graphs:}
We show how to perform a more efficient (and equivalent) contraction by operating on the node features $f$ by our formulation~\eqref{eq:dense_multicut} for the particular case of $s(\cdot, \cdot)$ defined as
\begin{equation}
    \label{eq:inner_product_similarity}
    s(f_i, f_j) = \la f_i, f_j \ra.
\end{equation}
From now on, unless stated otherwise, our edge costs will be given by~\eqref{eq:inner_product_similarity}.

\begin{lemma}[Contraction with node features]
\label{lemma:contraction_node_features}
Assume edge costs are measured by~\eqref{eq:inner_product_similarity} and nodes $i$ and $j$ are contracted to $m$.
Then features of node $m$ given by
\begin{equation}
    \label{eq:features_aggregation}
    f_m = f_i + f_j
\end{equation}
produce contracted edge costs according to~\eqref{eq:contraction_gaec}.
\end{lemma}
\begin{proof}
By applying~\eqref{eq:inner_product_similarity} for $l \in V$ and comparing with~\eqref{eq:contraction_gaec} we get
\begin{multline*}
    s(f_m, f_l) = \la f_m, f_l \ra 
    = \la f_i, f_l \ra + \la f_j, f_l \ra \\
    = s(f_i, f_l) + s(f_j, f_l)\,. \qedhere
\end{multline*}
\end{proof}

Next we will build on the previous result to devise heuristics for solving dense multicut problem~\eqref{eq:dense_multicut} efficiently.
\paragraph{GAEC for complete graphs:}
We devise an algorithm which exactly imitates GAEC~\cite{keuper2015efficient} but is applicable to our formulation on complete graphs~\eqref{eq:dense_multicut}.
Specifically to make GAEC efficient with node features and a complete graph, we sparsify the original graph $G$ by working on its directed $k$-nearest neighbours (NN) graph $(V, \A)$.
The NN graph stores candidate edges for contraction.
The arc set $\A$ is populated by nearest neighbour search w.r.t.\ feature similarity~\eqref{eq:inner_product_similarity} and is updated on each edge contraction.
We denote outgoing neighbours of $i$ as $\N^+_i = \{l \vert (i, l) \in \A\}$ and similarly $\N^-_i$ as incoming neighbours.
We write $\N_i^+ \cup \,\N_j^+$ as $\N^+_{ij}$ and similarly for incoming neighbours. 
Lastly the set $\argtopk_{s \in S} g(s)$ contains the $k$ elements of $S$ having the largest values of $g(s)$.
The complete strategy to obtain a feasible solution of dense multicut problem is described in Algorithm~\ref{alg:dense-multicut-gaec}.
It imitates Algorithm~\ref{alg:gaec} by iteratively searching and contracting the most attractive edge, but it restricts its search only to the NN graph thereby reducing computation. After contraction, the NN graph is updated (lines~\ref{alg:add_merged_node}-\ref{alg:dense_gaec_contraction_complete}) by only recomputing nearest neighbors of nodes which were affected by the contraction in the NN graph.
\begin{algorithm2e}
\newcommand\mycommfont[1]{\footnotesize\ttfamily\textcolor{blue}{#1}}
\SetCommentSty{mycommfont}
\caption{\texttt{Dense GAEC}}
\label{alg:dense-multicut-gaec}
\KwData{Node features $f_i, \forall i \in V$; Number of nearest neighbours $k$}
\KwResult{
Clusters $V$
}

\tcp{Find nearest neighbours of each node}
$\A = \{(i, j) \vert i \in V, j \in \argtopk_{i' \neq i} \la f_i, f_{i'} \ra \}$ \;\label{alg:initial_nn_search}
\While{$\max_{(u, v) \in \A} \la f_u, f_v \ra \geq 0$}
{
$m \coloneqq (i, j) = \arg\max_{(u, v) \in \A} \la f_u, f_v \ra$\;
\tcp{Update nodes}
$f_m = f_i + f_j$\;\label{alg:aggregate_features}
$ V = (V \cup m) \setminus \{i, j\}$ \;\label{alg:add_merged_node}
\tcp{Update nodes having $i, j$ as NN}
$H = \{(q, r) \vert q \in \N_{ij}^-, r \in \argtopk_{l \in V \setminus q} \la f_q, f_l\ra\}$\;\label{alg:gaec_update_nn_q}
\tcp{NN of merged node}
$H = H \cup \{(m, r) \vert r \in \argtopk_{l \in V \setminus m} \la f_m, f_l\ra\}$\;\label{alg:gaec_update_nn_m}
\tcp{Add arcs and remove arcs with $i, j$}
$\A = (\A \cup H) \setminus (\{(\cdot, i)\} \cup \{(\cdot, j)\} \cup \{(i, \cdot)\} \cup \{(j, \cdot)\})$\;
\label{alg:dense_gaec_contraction_complete}
}
\end{algorithm2e}

\begin{proposition}[Dense Greedy Contraction]
\label{prop:dense_greedy_contraction}
Algorithm~\ref{alg:dense-multicut-gaec} always merges a pair of nodes $i$ and $j$ with the largest edge cost i.e.,
\begin{equation}
\label{eq:arcset_optimal}
(i, j) \in \argmax_{(u, v) \in \A} \la f_u, f_v \ra \implies \la f_i, f_j \ra \geq \max_{u, v \neq u} \la f_u, f_v \ra.
\end{equation}
\begin{proof}
The statement is trivially satisfied before any merge operation is performed since $\A$ is constructed by nearest neighbour search over all nodes in line~\ref{alg:initial_nn_search} of the algorithm. 
We now show that after each merge operation (i.e., after line~\ref{alg:dense_gaec_contraction_complete} of the algorithm) the statement~\eqref{eq:arcset_optimal} still holds. 
We define $Q = m \cup \N_{ij}^-$. Two cases can arise:
\paragraph{Case 1: $\{i, j\} \cap Q \neq \varnothing$:}
Due to nearest neighbour search for all nodes in $Q$ at lines~\ref{alg:gaec_update_nn_q} and~\ref{alg:gaec_update_nn_m}, the statement holds.
\paragraph{Case 2: $\{i, j\} \cap Q = \varnothing$:}
In this case if $i$ is the contracted node $m$ from the last edge contraction operation then $(i, j) \in \A$ due to line~\ref{alg:gaec_update_nn_q}. 
If $i \neq m$ then it remains connected to its nearest neighbours either due to the initial NN search at line~\ref{alg:initial_nn_search} or the NN update at lines~\ref{alg:gaec_update_nn_q} and~\ref{alg:gaec_update_nn_m}.
\end{proof}
\end{proposition}
Note that the claim of Prop.~\ref{prop:dense_greedy_contraction} does not ensure that the arcset $\A$ will contain all nearest neighbour arcs after contraction. Instead it guarantees that the most attractive edge will always be present in the nearest neighbour graph, foregoing the need to search in the complete graph.  
This proves that the Algorithm~\ref{alg:dense-multicut-gaec} performs locally optimal merges as proposed in~\cite{keuper2015efficient} and is also scalable to large complete graphs. As a downside the algorithm requires costly nearest neighbour search after every edge contraction. Since computing nearest neighbours and contracting edges is not commutative, in the worst case one has to recompute the nearest neighbours on the contracted graph from scratch. 

\paragraph{Incremental nearest neighbours:} For faster nearest neighbour updates after edge contraction we show how to reuse more of the previously computed nearest neighbors through the following two approaches. 
First, for all nodes whose nearest neighbours are merging nodes (i.e., line~\ref{alg:gaec_update_nn_q} of Alg.~\ref{alg:dense-multicut-gaec}), we check if merged node $m$ is already a nearest neighbour without requiring exhaustive search. 
Specifically assume a contracting node $i$ was a $k$-nearest neighbour of some other node $q \in V \setminus i$.
Then the merged node $m$ is a $k$-nearest neighbour of $q$ if $\la f_q, f_m \ra \geq \min_{l \in \N_q^+}\la f_q, f_l \ra$.
This check can be cheaply performed for all such nodes thereby reducing computation.
Second, we devise a criterion which can allow to efficiently populate nearest neighbours of the contracted node $m$. 
\begin{proposition}[Incremental nearest neighbours]
Let the $k$-nearest neighbours $\N_i^+, \N_j^+$ of nodes $i$ and $j$ be given.
Assume that nodes $i$, $j$ are merged to form a new node $m$.
Then edge costs between nodes $v \in V \setminus \N_{ij}^+$ and $m$ are bounded from above by
\begin{equation*}
    b_{ij} \coloneqq \min_{p \in \N_i^+} \la f_i, f_p\ra + \min_{q \in \N_j^+} \la f_j, f_q\ra
\end{equation*}
\label{prop:incremental_nn_bound}
\end{proposition}
\begin{proof}
Since neighbours of $i$ are computed by nearest neighbours search we have for all nodes $p' \notin \N_i^+$
\begin{equation*}
   \la f_i, f_{p'} \ra \leq \min_{p \in \N_i^+} \la f_i, f_p \ra,
\end{equation*}
and similarly for node $j$. Then by definition of $v$ and Lemma~\ref{lemma:contraction_node_features} we obtain
\begin{align*}
    \la f_m, f_v \ra &= \la f_i, f_v \ra + \la f_j, f_v \ra \\
    &\leq \min_{p \in \N_i^+} \la f_i, f_p\ra + \min_{q \in \N_j^+} \la f_j, f_q\ra\,. \qedhere
\end{align*}
\end{proof}

\begin{figure}
    \centering
    \scalebox{1.5}{\pgfdeclarelayer{background}
\pgfsetlayers{background,main}
\def\connect#1#2#3{%
    \path let
      \p1 = ($(#2)-(#1)$),
      \n1 = {veclen(\p1)},
      \n2 = {atan2(\y1,\x1)} 
    in
      (#1) -- (#2) node[#3, midway, rotate = \n2, shading angle = \n2+90, minimum        width=\n1, minimum height=1pt, inner sep=1pt] {};
}

\begin{tikzpicture}
\definecolor{scattercolor1}{HTML}{800080}
\definecolor{scattercolor2}{HTML}{DD0000}
\definecolor{scattercolor3}{HTML}{009900}
\definecolor{scattercolor4}{HTML}{808000}
\definecolor{scattercolor5}{HTML}{D30E78}
\definecolor{scattercolor6}{HTML}{036ffc}
\definecolor{scattercolor7}{HTML}{eb9800}
\tikzstyle{vertex}=[circle, draw, ultra thin, fill=white, inner sep=0.2pt, minimum width=8pt]
\tikzstyle{vector}=[rectangle, draw,inner sep=1pt,outer sep=0pt,
minimum height=5mm, minimum width=2mm, ultra thin]

\def\maxX{3}
\foreach \x [count=\n] in {0,...,\maxX}{
    \foreach \y in {0,...,2}{
    \ifthenelse{1 = \x \AND 1 = \y}
    {\node[draw=none](\x\y) at (\x, \y) { };}
    {
    \ifthenelse{2 = \x \AND 1 = \y}
    {\node[draw=none](\x\y) at (\x, \y) { };}
    \node[vertex, fill=white, outer sep = 1mm](\x\y) at (\x, \y) { };}
    }
}

\node[vertex, fill=scattercolor5!50, thin, inner sep=1.7pt, outer sep = 0mm] at (1.15, 1){{\small{$i$}}};
\node[vertex, fill=scattercolor6!50, thin, inner sep=1pt, outer sep = 0mm] at (1.85, 1){\small{$j$}};

\node[draw=none, scattercolor5!50] at (-0.5, 2.0){\scriptsize{$\mathcal{N}_i^+$}};
\node[draw=none, scattercolor6!50] at (3.6, 2.0){\scriptsize{$\mathcal{N}_j^+$}};
\node[draw=none, scattercolor4!50] at (2.6, -0.0){\scriptsize{$\mathcal{N}_{ij}^-$}};


\begin{pgfonlayer}{background}
\draw[black, ->] (0.98, 0.88) -- (0.1, 0.1);
\draw[black, <-] (0.98, 0.88) -- (0.1, 0.1);
\draw[black, ->] (1.15, 1) -- (0.1, 1.9);
\draw[black, ->] (1.15, 1) -- (0.14, 1);
\draw[black, ->] (1.15, 1) -- (1, 1.86);

\draw[black, ->] (1.85, 1) -- (1.1, 1.9);
\draw[black, ->] (1.85, 1) -- (2, 1.86);
\draw[black, ->] (1.85, 1) -- (2.9, 1.9);

\draw[black, ->] (1, 0.1) -- (1.15, 0.82);
\draw[black, ->] (2, 0.1) -- (1.85, 0.80);

\connect{11}{21}{top color=scattercolor5!50, bottom color=scattercolor6!50};

\draw[draw=black, fill=scattercolor5!50,opacity=0.2](02.north) to[closed,curve through={(02.north west) .. (02.west) .. (01.west) .. (00.west) .. (00.south) .. (00.east) ..(00.north east) .. (01.east) .. (02.south east) .. (12.south).. (12.east)}](02.north);

\draw[fill=scattercolor6!50,opacity=0.2](12.west) to[closed,curve through={(12.north) .. (22.north) .. (32.north) .. (32.east) .. (32.south) .. (22.south) .. (12.south)}](12.west);

\draw[fill=scattercolor4!50,opacity=0.2](00.west) to[closed,curve through={(00.north) .. (10.north) .. (20.north) .. (20.east) .. (20.south) .. (10.south) .. (00.south)}](00.west);

\end{pgfonlayer}

\end{tikzpicture}}
    \caption{Illustration of nearest neighbour graph and an edge $ij$ being contracted.
    The set $\N_{ij}^+ = \N_i^+ \cup \N_j^+$ is searched first to find nearest neighbours of the merged node efficiently (Prop.~\ref{prop:incremental_nn_bound}).
    The nodes in set $\N_{ij}^-$ need to update their nearest neighbours since their current nearest neighbour nodes $i$, $j$ are getting contracted.
    Only the arcs to/from $i$, $j$ are shown. }
    \label{fig:knn-graph-illustration}
\end{figure}
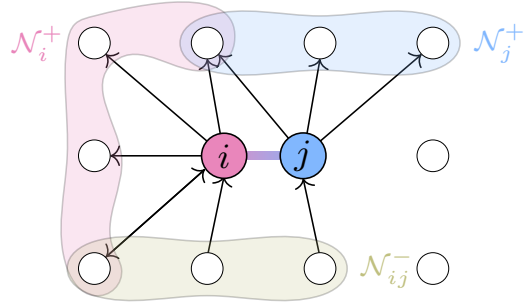

The above proposition gives an upper bound of feature similarity (i.e., edge cost) of merged node $m$ with all nodes not in $\N_{ij}^+$. Thus if a node in $\N_{ij}^+$ exceeds this upper bound it is more similar to $m$ than all nodes not in $\N_{ij}^+$. This allows to possibly skip recomputing the nearest neighbors of $m$ in Alg.~\ref{alg:dense-multicut-gaec} (line~\ref{alg:gaec_update_nn_m}). 
\begin{lemma}
If
\begin{equation}
\abs{ \{p \in \N_{ij}^+ : \la f_m, f_p \ra \geq b_{ij} \}} \geq k
\end{equation}
then $k$-nearest neighbour of node $m$ given by $\argtopk_{v \in V \setminus \{i, j, m\}} \la f_m, f_v \ra $ can be chosen as  
$\argtopk_{p \in \N_{ij}^+} \la f_m, f_p \ra $.
\begin{proof}
Since the elements of $\N_{ij}^+$ already satisfy the bound $b_{ij}$ from Prop.~\ref{prop:incremental_nn_bound} and there are at least $k$ many such elements, the $k$-nearest neighbours of node $m$ can be taken from $\N_{ij}^+$. 
\end{proof}
\end{lemma}
Both of these approaches for efficiently updating the NN graph after contraction are used in Alg.~\ref{alg:inc-nn}. Additionally in Alg.~\ref{alg:inc-nn} we skip exhaustive search if a node still has $p$-many nearest neighbours where $p \in [1, k)$.
Algorithm~\ref{alg:inc-nn} can be used instead of lines~\ref{alg:gaec_update_nn_q} and~\ref{alg:gaec_update_nn_m} in Alg.~\ref{alg:dense-multicut-gaec} for improved performance.
See Figure~\ref{fig:knn-graph-illustration} for an illustration on nearest neighbour graph and edge contraction update. 
\begin{algorithm2e}
\newcommand\mycommfont[1]{\footnotesize\ttfamily\textcolor{blue}{#1}}
\SetCommentSty{mycommfont}
\caption{\texttt{Incremental NN update}}
\label{alg:inc-nn}
\KwData{
Contracting nodes $i$, $j$;
Contracted node $m$;
NN graph $(V, \A)$;
Node features $f_i, \forall i \in V$;
Num. of neighbours $k$;
}
\KwResult{
Nearest neighbour arcs $H$ to add in $\A$
}
\tcp{NNs of $m$ by Prop.~\ref{prop:incremental_nn_bound}}
$H = \{(m, l) \vert l \in \N_{ij}^+, \la f_m, f_l \ra \geq b_{ij}\}$\;
\tcp{Keep at most $k$ NN}
$H = \argtopk_{(m, l) \in H} \la f_m, f_l \ra$\;
\If{$H = \varnothing$}
{
    $H = \{(m, r) \vert r \in \argtopk_{l \in V \setminus m} \la f_m, f_l\ra\}$\;
    \label{alg:full_search_m}
}
\For{$q \in \N_{ij}^- \setminus \{i, j\}$}
{
    \tcp{Check if $m$ a NN of $q$}
    \If{$\la f_q, f_m \ra \geq \min_{l \in \N_q^+} \la f_q, f_l \ra$}
    {
        $H = H \cup (q, m)$\;
    }
    \Else
    {
    $H = H \cup \{(q, r) \vert r \in \argtopk_{l \in V \setminus q} \la f_q, f_l\ra) \}$\;\label{alg:full_search_q}
    }
}
\end{algorithm2e}

\subsection{Lazy Edge Contraction}
\label{sec:lazy_edge_contraction}
We further forego the need for nearest neighbours recomputation after edge contraction by lifting the restriction of performing only greedy moves. 
This allows to maximally utilize the NN graph: the algorithm performs contractions, including non-greedy ones, until no contraction candidates are present in the NN graph. Specifically we do not perform the exhaustive search in lines~\ref{alg:full_search_m} and~\ref{alg:full_search_q} of Alg.~\ref{alg:inc-nn} and only return the nearest neighbours which are easily computable. The NN graph is repopulated as lazily as possible i.e., when no contraction candidates are left. 
In addition to being more efficient this strategy is reminiscent of the balanced edge contraction approach of~\cite{kardoost2018solving}. The authors normalized the edge costs with cluster size of two end-points. These normalized edge costs were used to find the edge to contract. This strategy encouraged consecutive contractions to occur at different regions of the graph. As our lazy approach does not always make the nearest neighbours of the contracted node available thus contractions can only be done to nodes other than the contracted node. This also produces contractions in different regions. 

Lastly we explore efficient methods for approximate nearest neighbour search~\cite{malkov2018efficient_hnsw} for populating the initial NN graph. For later searches we still use exact methods as the search space is reduced due to contractions.

\subsection{Varying Affinity Strength}
Our basic edge costs computed by $\la f_i, f_j \ra$ for two features $f_i$ and $f_j$ have one fundamental limitation: 
Clusters will by default occupy whole quadrants. 
In other words, whenever two features have angle lower than $90^\circ$ they are attractive and will prefer to be in the same cluster, see Figure~\ref{fig:angle-cluster-illustration}.
In order to let our formulation favor larger or smaller clusters, we modify our original similarity function $s(\cdot,\cdot)$ by adding an additional term indicated by $\alpha$-variables:
\begin{align}
    \overline{f}_i &= [f_i; \alpha_i], \\
    s(\overline{f}_i,\overline{f}_j) &= \la f_i, f_j \ra \pm \alpha_i \cdot  \alpha_j\,,
    \label{eq:extended_features_similarity}
\end{align}
where we choose positive sign for favoring larger clusters and negative for smaller clusters.
In our experiments we will set $\alpha_i = \alpha > 0$, with $-$ in~\eqref{eq:extended_features_similarity} to prefer many small sized clusters.
Moreover, we note that our contraction mechanism carries over directly to this extended setting.
\begin{lemma}
\label{lemma:features_aggregation_appended}
Aggregating features of the contracted node $m$ by $\overline{f}_m = \overline{f}_i + \overline{f}_j$ is equivalent to setting edge costs as per~\eqref{eq:contraction_gaec} on complete graph.
\begin{proof}
Similar to the proof of Lemma~\ref{lemma:contraction_node_features} as follows
\begin{align*}
    s(\overline{f}_m, \overline{f}_l) &= 
    \la f_m, f_l \ra \pm \alpha_m\cdot\alpha_l \\
    &=\la f_i + f_j, f_l \ra \pm (\alpha_i + \alpha_j)\cdot\alpha_l \\
    &= \la f_i, f_l \ra \pm \alpha_i\cdot\alpha_l +
        \la f_j, f_l \ra \pm \alpha_j\cdot\alpha_l \\
    &= s(\overline{f}_i, \overline{f}_l) + s(\overline{f}_j, \overline{f}_l)\,.\qedhere
\end{align*}
\end{proof}
\end{lemma}

\paragraph{Large clusters:}
For preferring larger clusters (corresponding to choosing $+$ in~\eqref{eq:extended_features_similarity}), we work directly on the extended feature set $\overline{f}_i = [f_i; \alpha_i]$ and use it in the NN graph.

\paragraph{Small clusters:}
For preferring smaller clusters (corresponding to choosing $-$ in~\eqref{eq:extended_features_similarity}), we must modify our algorithms slightly.
In order to construct NN graphs we will use two sets of features:
First, the query nodes will have their features defined by $\hat{f}_i = [f_i, -\alpha_i]$ and
second, the pre-existing nodes $j \in V$ in the graph will keep the same features $\overline{f}_j$ from~\eqref{eq:extended_features_similarity}.
To search for nearest neighbors of node $i$ in the graph $V$ the modified similarity function~\eqref{eq:extended_features_similarity} can be implemented by an inner product as
\begin{equation}
    s(\overline{f}_i, \overline{f}_j) = \la \hat{f}_i, \overline{f}_j \ra\,.
    \label{eq:extended_features_similarity_ip}
\end{equation}

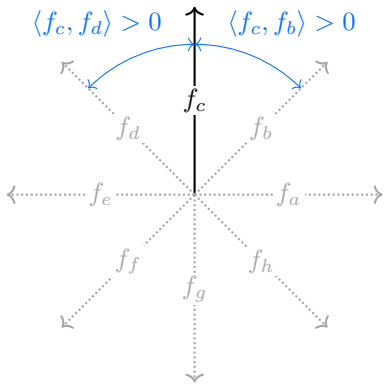
\begin{figure}
    \centering
    \scalebox{1.0}{\begin{tikzpicture}
\definecolor{scattercolor1}{HTML}{800080}
\definecolor{scattercolor2}{HTML}{DD0000}
\definecolor{scattercolor3}{HTML}{009900}
\definecolor{scattercolor4}{HTML}{808000}
\definecolor{scattercolor5}{HTML}{D30E78}
\definecolor{scattercolor6}{HTML}{036ffc}
\definecolor{scattercolor7}{HTML}{eb9800}
\tikzstyle{normal-edge}=[gray!70, thick, densely dotted]

\draw[->, black, thick] (0, 0) -- ({90}:2.5);
\foreach \N in {2,...,8}
  \draw[->, style=normal-edge] (0, 0) -- ({360/8*(\N-1)+90}:2.5);

\foreach \N/\xtext in {2/d,3/e,4/f,5/g,6/h,7/a, 8/b}{
  \node[fill=white, inner sep = 1pt] (t\N) at ($(0,0)!0.5!({360/8*(\N-1)+90}:2.5)$) {\textcolor{gray!70}{${f_\xtext}$}};
}

\node[fill=white, inner sep = 1pt] (t1) at ($(0,0)!0.5!({90}:2.5)$) {\textcolor{black}{${f_c}$}};

\draw [<->, scattercolor6] (45:2.0) arc (45:90:2.0) node [above right,pos=0.8] {{$\langle f_c, f_b \rangle > 0$}};

\draw [<->, scattercolor6] (90:2.0) arc (90:135:2.0) node [above left,pos=0.2] {{$\langle f_c, f_d \rangle > 0$}};


\end{tikzpicture}}
    \caption{Illustration of edge costs between $8$ nodes where feature vectors of each node $i$ is in two-dimensional space i.e., $f_i \in \R^2$. If we want each node to be a separate cluster then the edge costs measured by~\eqref{eq:inner_product_similarity} are not suitable. This is because there will always be atleast two vectors with positive costs preferring to be in the same cluster. Using a large enough positive value of $\alpha$ with a $-$ in ~\eqref{eq:extended_features_similarity} this issue can be resolved.
    }
    \label{fig:angle-cluster-illustration}
\end{figure}

\subsection{Computational complexity}
Our basic formulation~\eqref{eq:dense_multicut} with $\alpha=0$ in~\eqref{eq:extended_features_similarity} is shown to be solvable in time $\mathcal{O}(\abs{V}^{d^2})$ in~\cite{veldt2017correlation_lowrank} through zonotope vertex enumeration~\cite{onn2001vector_partitioning, stinson2016randomized_zonotope_enumeration}. These results are also applicable when choosing $+$ in~\eqref{eq:extended_features_similarity}. In our experiments having $-$ in~\eqref{eq:extended_features_similarity} is vital for obtaining a good clustering, in which case the results of~\cite{veldt2017correlation_lowrank} are not directly transferable.

\section{Experiments}
\label{sec:experiments}
We study the benefits of the multicut on complete graphs~\eqref{eq:dense_multicut} and compare possible algorithms on the tasks of ImageNet~\cite{deng2009imagenet} clustering and Cityscapes~\cite{cordts2016cityscapes} panoptic segmentation. All datasets are made available in~\cite{swoboda2022structured}.
The algorithms are
\begin{description}
\item[\alg{GAEC}:]
The greedy additive edge contraction algorithm from~\cite{keuper2015efficient}(Alg.~\ref{alg:gaec}) is run on the complete graph where all edge costs are precomputed and then passed to the algorithm.
\item[\alg{RAMA}:]
We also compare with the recent GPU-based multicut solver of~\cite{rama_22}. Similar to \alg{GAEC} we run it on the complete graph. The solver uses dual optimization for better solution quality and also gives lower bounds to the multicut objective~\eqref{eq:multicut}. As a drawback it cannot handle large instances due to high memory requirement of complete graphs. We evaluate on NVIDIA A40 GPU with 48GB of memory. 
\item[\alg{DGAEC}:]
Our Algorithm~\ref{alg:dense-multicut-gaec} which operates on node features and performs contractions according to Lemma~\ref{lemma:contraction_node_features}. The nearest neighbour graph is updated by exhaustive search after edge contraction. The number of nearest neighbours $k$ is set to $1$, a larger value does not benefit since Prop.~\ref{prop:incremental_nn_bound} is not utilized.
\item[\alg{DGAECInc}:]
Our Algorithm~\ref{alg:dense-multicut-gaec} which additionally makes use of Alg.~\ref{alg:inc-nn} for incremental neighbour updates after edge contraction. The value of $k$ is set to $5$. 
\item[\alg{DLAEC}:]
A variant of our \alg{DGAECInc} where non-greedy moves are also allowed as described in Sec.\ \ref{sec:lazy_edge_contraction}.
\item[\alg{DAppLAEC}:]
Another variant of our \alg{DLAEC} where initial nearest neighbours are computed by approximate nearest neighbour search method~\cite{malkov2018efficient_hnsw} through library of~\cite{johnson2019_faiss}. 
\end{description}
For all multicut algorithms on all datasets we set the value of affinity strength $\alpha_i$ in ~\eqref{eq:extended_features_similarity_ip} to $0.4$, preferring small clusters. We do not compare with the randomized algorithm of~\cite{veldt2017correlation_lowrank} since it does not account for affinity strength with preference on smaller clusters. All CPU algorithms are run on an AMD 7502P CPU with a maximum of 16 threads to allow for faster nearest neighbour search. 
\subsection{ImageNet clustering}
We evaluate clustering of the ImageNet~\cite{deng2009imagenet} validation set containing $50k$ images.
Each image in the dataset acts as a node for our dense multicut formulation. The features of each image are computed by a ResNet50~\cite{he2016resnet} backbone trained by \mbox{MoCov3}~\cite{chen2021mocov3} in unsupervised fashion by a constrastive loss on the training split of ImageNet.
The features have a dimension of $2048$ and are normalized to have unit $L_2$ norm.
We create two problem instances containing $5k$ and $50k$ images by considering $100$ and all $1000$ classes respectively. 

\paragraph{Clustering quality:}
Before comparing our algorithmic contributions we first test the efficacy of our dense multicut formulation by comparing its clustering result with $k$-means~\cite{lloyd_kmeans} using the implementation from~\cite{scikit-learn} and initialization of~\cite{arthur07_kmeans++}. Since $k$-means requires the number of clusters to be known beforehand we set it to the number of classes in the problem instance. For an additional comparison we also run $k$-means on the number of clusters given by our dense multicut algorithm. The quality of clustering results are evaluated using normalized mutual information (NMI) and adjusted mutual information (AMI) metrics~\cite{vinh10a_nmi}. The results are given in Table~\ref{tab:imagenet_clustering}. 
We observe that although our formulation does not require the number of clusters to be specified, the results are on par with $k$-means. Additionally the value of affinity strength $\alpha$ does not need to be changed for different problem instances. As compared to $k$-means our algorithms are much faster especially on the larger instance. The \alg{RAMA} solver of~\cite{rama_22} performs better than all other approaches on the smaller instance but runs out of memory for the larger one.
Lastly, our formulation creates more clusters than the number of classes. This is mainly due to presence of outliers in the feature space as the feature extractor is trained without any groundtruth information.  

\paragraph{Algorithms comparison:}
We compare different algorithms for solving dense multicut problem~\eqref{eq:dense_multicut} for imageNet clustering in Table~\ref{tab:imagenet_algorithms}. Firstly, we see that on the smaller instance the GPU based solver \alg{RAMA}~\cite{rama_22} gives the best performance. Secondly using incremental nearest neighbour search through Alg.~\ref{alg:inc-nn} gives better run time than exhaustive search. Lastly our non-greedy algorithms give the best run time among all CPU-based algorithms although with slightly worse objectives. 

On the smaller instance, \alg{RAMA} outperforms other algorithms in terms of the objective value~\eqref{eq:dense_multicut} and also gives better clustering quality as compared to $k$-means.
As a drawback \alg{RAMA} cannot handle large dense multicut instances.
This shows multicut on complete graphs can be a suitable alternative to $k$-means.
We speculate that algorithmic improvements on top of our proposed algorithms will further improve clustering quality for large graphs.

\begin{table}
\caption{Quality of clustering on ImageNet validation set. t [s]: compute time in seconds, NMI: normalized mutual information, AMI: adjusted mutual information, \# clusters: number of clusters, $\dagger$: out of GPU memory. For $k$-means the number of clusters was specified as input. }
\label{tab:imagenet_clustering}
\begin{center}
\begin{tabular}{l r r r r}
\toprule
Method & t [s] $\downarrow$ & NMI $\uparrow$ & AMI $\uparrow$ & \# clusters \\
\midrule 
\multicolumn{5}{c}{\textit{ImageNet-$100$} $(\abs{V} = 5k)$} \\
\midrule
$k$-means & 16 & 0.42 & 0.27 & 100 \\ 
$k$-means & 32 & 0.53 & 0.26 & 333 \\ 
\alg{RAMA} & \textbf{0.9} & \textbf{0.57} & \textbf{0.29} & 639 \\ 
\alg{DGAECInc} & 42 & 0.43 & 0.22 & 343 \\ 
\alg{DAppLAEC} & 3.2 & 0.47 & 0.26 & 333 \\ 

\midrule 
\multicolumn{5}{c}{\textit{ImageNet-$1000$} $(\abs{V} = 50k)$} \\
\midrule
$k$-means & 701 & 0.54 & 0.2 & 1000 \\ 
$k$-means &  1801 & \textbf{0.61} & 0.19 & 2440 \\ 
\alg{RAMA} & $\dagger$ & $\dagger$ & $\dagger$ & $\dagger$ \\ 
\alg{DGAECInc} & 2964 & 0.49 & 0.19 & 2488 \\ 
\alg{DAppLAEC} & \textbf{65} & {0.56} & \textbf{0.26} & 2440 \\ 

\bottomrule
\end{tabular}
\end{center}
\end{table}    

\begin{table}[h]
\caption{Comparison of algorithms for solving dense multicut problem on two splits of Imagenet validation set. t [s]: compute time in seconds, Obj: objective value of clustering~\eqref{eq:dense_multicut}, $\dagger$: out of GPU mem.\, $\star$: no result within $3$ hours.}
\label{tab:imagenet_algorithms}
\begin{center}
\begin{tabular}{l r r r r r}
\toprule
 & \multicolumn{2}{c}{\textit{ImageNet-$100$}} & \multicolumn{2}{c}{\textit{ImageNet-$1000$}} \\
\cmidrule(lr){2-3} \cmidrule(lr){4-5}
Method & t [s] $\downarrow$ & Obj $\downarrow$ & t [s] $\downarrow$ & Obj $\downarrow$ \\
 \midrule
\alg{GAEC} & 4.5 & -{6.84e5} & 552 & -\textbf{9.353e7} \\
\alg{RAMA} & \textbf{0.9} & -\textbf{6.95e5} & $\dagger$ & $\dagger$ \\
\alg{DGAEC} & 132 & -{6.84e5} & $\star$ & $\star$ \\ 
\alg{DGAECInc} & 42 & -{6.84e5} & 2934 & -\textbf{9.353e7} \\ 
\alg{DLAEC} & 5 & -6.83e5 & 341 & -9.332e7 \\ 
\alg{DAppLAEC} & {3.2} & -6.83e5 & \textbf{65} & -9.332e7 \\ 
\bottomrule
\end{tabular}
\end{center}
\end{table}

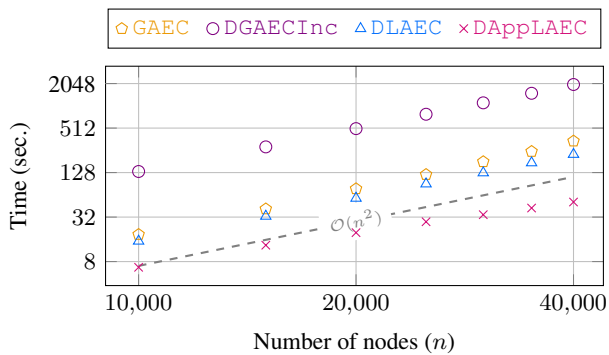
\begin{figure}[h]
\begin{center}
\begin{tikzpicture}[font=\small]
\begin{axis}[
        width=8.2cm,
        height=4.5cm,
        xmajorgrids,
        ymajorgrids,
        xlabel={Number of nodes ($n$)},
        ylabel={Time (sec.)},
        ylabel near ticks,
        xlabel near ticks,
        xmode=log,
        log basis x = {2},        
        ymode=log,
        log basis y = {2},
        xmin = 9000,
        xmax = 44000,
        mark size=0.5ex,
        log ticks with fixed point,
        legend pos=north east,
        clip marker paths=true,
        xtick={5000, 10000, 20000, 40000},
        scaled y ticks=false,
        ytick={8, 32, 128, 512},
        extra y ticks={2048},
        extra y tick labels={2048},
        legend style={at={(0.5,1.08)},
        legend cell align={center},
	    anchor=south,legend columns=5},
    ]
    
    \draw[dashed, gray, thick] (10000, 7) -- node[fill=white,midway,sloped,font=\tiny]{$\mathcal{O}(n^2)$} (40000, 112);

    
    \addlegendimage{only marks, mark=pentagon, mark options={color=scattercolor7}}
    \addlegendentry{\textcolor{scattercolor7}{\normalfont\texttt{GAEC}\enspace}}

    \addlegendimage{only marks, mark=o, mark options={color=scattercolor1}}
    \addlegendentry{\textcolor{scattercolor1}{\normalfont\texttt{DGAECInc}\enspace}}

    \addlegendimage{only marks, mark=triangle, mark options={color=scattercolor6}}
    \addlegendentry{\textcolor{scattercolor6}{\normalfont\texttt{DLAEC}\enspace}}

    \addlegendimage{only marks, mark=x, mark options={color=scattercolor5}}
    \addlegendentry{\textcolor{scattercolor5}{\normalfont\texttt{DAppLAEC}}}
    
    \addplot[only marks, mark = pentagon, color=scattercolor7] table [
        col sep=comma, 
        skip first n=1,
        x index=0, y index=1
    ] {figures/emp_complexity_analysis/inet_comp_cpu20_adj_matrix.csv};

    \addplot[only marks, mark = o, color=scattercolor1] table [
        col sep=comma, 
        skip first n=1,
        x index=0, y index=1
    ] 
    {figures/emp_complexity_analysis/inet_comp_cpu20_gaec_inc_3.csv};

    \addplot[only marks, mark = triangle, color=scattercolor6] table [
        col sep=comma, 
        skip first n=1,
        x index=0, y index=1
    ] {figures/emp_complexity_analysis/inet_comp_cpu20_laec_3.csv};

    \addplot[only marks, mark = x, color=scattercolor5] table [
        col sep=comma, 
        skip first n=1,
        x index=0, y index=1
    ] {figures/emp_complexity_analysis/inet_comp_cpu20_laec_bf_later_3.csv};

\end{axis}
\end{tikzpicture}
\caption{Runtime comparison of our algorithms on instances of varying sizes taken from ImageNet val.\ set by varying number of classes. Both axes are in log-scale. Our algorithms \alg{DGAECInc} and \alg{DLAEC} show quadratic complexity. Our algorithm \alg{DAppLAEC} behaves subquadratically benefiting from approximate nearest neighbour search.
}
\label{fig:emp_comp_analysis}
\end{center}
\end{figure}

Lastly we perform empirical time complexity analysis of our algorithms showing quadratic and subquadratic behaviour in Figure~\ref{fig:emp_comp_analysis}. More details are provided in the Appendix. 
\subsection{Panoptic segmentation}
We evaluate our method on the task of panoptic segmentation~\cite{kirillov2019panoptic_quality} on the Cityscapes dataset~\cite{cordts2016cityscapes}. The panoptic segmentation task consists of assigning a class label to each pixel and partitioning different instances of classes with object categories (e.g. car, person etc.). We focus on the task of partitioning for which the multicut formulation~\eqref{eq:multicut} has been used by~\cite{kirillov2017instancecut,abbas2021cops}. The latter work used a carefully crafted graph structure. Our dense multicut~\eqref{eq:dense_multicut} formulation foregoes the need for finding a suitable graph structure. 
We use the pretrained Axial-ResNet50~\cite{wang2021max} network from~\cite{yu2022kmeans_mask_transformer}, made available by~\cite{deeplab2_2021} to compute the node features. Specifically, the network computes $L_2$-normalized, $128$-dimensional and $4\times$ downsampled features in its intermediate stages which we use for our study without any training. 

For our evaluation we first compute semantic class predictions and then create a dense multicut instance for each semantic category with objects (i.e., car, person etc.). Such classes are also known as \textit{thing} classes. 
The goal of the multicut problem is then to partition all nodes belonging to same semantic class to different objects. This strategy creates a total of $1631$ dense multicut problem instances of varying sizes from $500$ images of the Cityscapes validation set. The largest problem instance contains around $43k$ nodes. To upsample the clustering back to original image resolution we interpolate the node features back to input image resolution. Afterwards each upsampled node is assigned to the cluster whose mean feature embedding is most similar.  

\paragraph{Clustering quality:}
\begin{table}
\caption{Comparison of panoptic segmentation on Cityscapes dataset. Multicut on sparse graph of~\cite{abbas2021cops} is computed by Alg.~\ref{alg:gaec}. For dense multicut we use the \alg{DAppLAEC} algorithm. $PQ_{th}$: Average panoptic quality of all \textit{thing} classes.}
\label{tab:cityscapes_panoptic_quality}
\begin{center}
\begin{tabular}{l r r}
\toprule
& \multicolumn{2}{c}{Panoptic quality $(\%) \uparrow$} \\
\cmidrule(lr){2-3}
Category & Sparse multicut & Dense multicut \\ 
\midrule
\textit{Person} & 40.0 & \textbf{46.9} \\
\textit{Rider} & 53.0 & \textbf{54.4} \\ 
\textit{Car} & 50.7 & \textbf{60.5} \\ 
\textit{Truck} & \textbf{52.7} & 52.3 \\ 
\textit{Bus} & \textbf{72.1} & 71.1 \\ 
\textit{Train} & \textbf{65.6} & 62.9 \\ 
\textit{Motorcycle} & \textbf{47.0} & 46.8 \\ 
\textit{Bicycle} & 45.7 & \textbf{46.9} \\ 
\midrule
$PQ_{th}$ & 53.3 & \textbf{55.2} \\ 
\bottomrule
\end{tabular}
\end{center}
\end{table}
As a first point of comparison we check whether formulating a multicut problem on the complete graph by~\eqref{eq:dense_multicut} is beneficial as compared to a handcrafted sparse graph structure.
We take the sparse graph structure from~\cite{abbas2021cops} as a baseline. Their graph also includes long-range edges for dealing with occlusions leading to about $10\cdot\abs{V}$ edges in total. We compute the edge costs in this sparse graph in the same way as for our dense formulation and use Alg.~\ref{alg:gaec} for computing multicut.

In Table~\ref{tab:cityscapes_panoptic_quality} we compare the quality of clustering through the panoptic quality metric~\cite{kirillov2019panoptic_quality}. 
We observe that our dense multicut formulation performs better than multicut on the sparse handcrafted graph.
This improvement is significant for classes which can have many instances of the same class within an image (i.e.\ person, car) thus making the partitioning problem difficult.
For classes with large objects (e.g.\ truck) having more edges does not help since the sparse graph can already capture most inter-pixel relations.
On average our dense multicut formulation gives better results than sparse multicut while alleviating the need for designing a graph structure. 

\paragraph{Algorithms comparison:}
\begin{table}[h]
\caption{Comparison of algorithms for solving dense multicut problem on Cityscapes validation set. (t [s]): average compute times in seconds, (Obj): average objective value of clustering~\eqref{eq:dense_multicut}.}
\label{tab:cityscapes_algorithms}
\begin{center}
\begin{tabular}{l r r }
\toprule
Method & t [s] $\downarrow$ & Obj ($\times 10^6$) $\downarrow$ \\
\midrule
\alg{GAEC} & 7.7 & -6.338 \\
\alg{DGAEC} & 84.1 & -6.338 \\ 
\alg{DGAECInc} & 3.2 & -6.338 \\ 
\alg{DLAEC} & 2.1 & -6.340 \\ 
\alg{DAppLAEC} & \textbf{1.5} & -\textbf{6.341} \\ 
\bottomrule
\end{tabular}
\end{center}
\end{table}
We compare dense multicut algorithms for the panoptic segmentation task in terms of objective value and run time.
We were not able to run \alg{RAMA} since the GPU could not store large graphs. The comparison of performance to the remaining algorithms averaged over all problem instances is given in Table~\ref{tab:cityscapes_algorithms}.

In terms of run time, we see that our most naive algorithm $\alg{DGAEC}$ is slower than $\alg{GAEC}$ which directly operates on edge costs. 
Our other algorithms surpass $\alg{GAEC}$ reaching up to an order of magnitude run time improvement with lazy edge contractions and approximate initial nearest neighbours search.
In terms of objective value we see slight improvement by our lazy contraction algorithms as compared to the greedy ones.


\paragraph{Sensitivity of affinity strength:}
In Table~\ref{tab:cityscapes_att_rep_strength} we study the effect of changing the value of $\alpha$ from~\eqref{eq:extended_features_similarity}. The results highlight that having $\alpha > 0$ is essential for good clustering quality.
Last, we see further improvement if the value of $\alpha$ is set differently for each semantic class. We refer to the Appendix for further results.
\begin{table}[h]
\caption{Results of panoptic segmentation via dense multicut with different values of attraction/repulsion strength $\alpha$ in~\eqref{eq:extended_features_similarity}. $PQ_{th}$: Avg.\ panoptic quality over all \textit{thing} classes.}
\label{tab:cityscapes_att_rep_strength}
\begin{center}
\begin{tabular}{l | r r r r r r r}
\toprule
$\alpha$ & $0.2$ & $0.3$ & $0.4$ & $0.5$ & $0.6$ & $0.7$ & $0.8$ \\ 
\midrule
$PQ_{th}$ & 54.5  & \textbf{55.8} & 55.2 & 55.0 & 54.1 & 52.0 & 49.3 \\
\bottomrule
\end{tabular}
\end{center}
\end{table}

\section{Conclusion}
\label{sec:conclusion}
We have demonstrated that optimizing multicut on large complete graphs is possible when using factorized edge costs through inner products of features.
We speculate that further algorithmic improvements are possible e.g.\ by performing dual optimization directly on the node features.

As a potential theoretical advantage our approach sidesteps the need for learning graph structure. This offers a possibility to embed it as a differentiable layer in neural networks, using e.g.\ the work~\cite{vlastelica2019differentiation}.

\section{Acknowledgements}
We are grateful to all reviewers especially reviewer TfhS for their helpful suggestions. We also thank Jan-Hendrik Lange for discussion about related works.

\bibliography{references}
\bibliographystyle{icml2023}

\newpage
\appendix
\onecolumn

\begin{center}
    \textbf{\Large Appendix}
\end{center}
\section{Time complexity analysis}
Theoretically all of our algorithms have asymptotic time complexity of $\mathcal{O}(d\cdot|V|^3)$. However, empirically we observe our algorithms show quadratic behaviour and get faster through Prop.~\ref{prop:incremental_nn_bound} and lazy contractions. We analyse empirical complexity of our algorithms in Figure~\ref{fig:emp_comp_analysis}. More details about worst-case complexity of our algorithms are as under
\begin{description}
\item[\alg{GAEC}:] In worst case scenario Alg.~\ref{alg:dense-multicut-gaec}, the set $\N_{ij}^-$ in line~\ref{alg:gaec_update_nn_q} can be $V \setminus \{i, j\}$. Therefore nearest neighbour search in line~\ref{alg:gaec_update_nn_q} has complexity $\mathcal{O}(d\cdot|V|^2)$ making each edge contraction operation quadratic. Overall complexity of the algorithm will then be $\mathcal{O}(d\cdot|V|^3)$.
\item[\alg{DGAEC}:] In worst case scenario Prop.~\ref{prop:incremental_nn_bound} might not help requiring exhaustive search for nearest neighbours. Thus asymptotic complexity remains same as \alg{DGAEC}.
\item[\alg{LAEC}:] Worst case complexity remains same as \alg{DGAEC}.
\item[\alg{DAppLAEC}:] Here we use approximate nearest neighbour search~\cite{malkov2018efficient_hnsw} but only to populate initial set of nearest neighbours i.e., the most costly operation. For later iterations we still use exhaustive search therefore asymptotic complexity remains the same as \alg{DLAEC}. 
\end{description}

Note that above complexity analysis assumes that the number of nearest neighbours $k$ is set to 1. This offers very limited potential for incremental nearest neighbour updates. A larger value of $k$ gives much speedup due to Alg.~\ref{alg:inc-nn}. A case distinction is provided below
\begin{description}
    
\item[$k=1$:]
Assume the $k$-nearest neighbour graph with $k = 1$ before edge contraction has the structure: $V = \{1, 2, ..., n\}$, $\mathcal{A} = \{(2, 1), (3, 1), ... (n, 1)\}$. Thus node $1$ is the nearest neighbour of all other nodes. If an edge containing node $1$ is contracted it will force all other nodes to recompute their nearest neighbours. Note that there are still $\mathcal{O}(|V|)$ many remaining nodes requiring nearest neighbour update. Due to this worst-case scenario time complexity of one edge contraction becomes quadratic making overall runtime cubic in the number of nodes.

\item[$k \gg 2$:]
Assume the node set after contracting an edge $ij$ is $V' := V\setminus \{i, j\}$. Then each node in $V'$ still has $k-2$ many nearest neighbours from within $V'$. In this case nearest neighbour queries only need to be performed between the merged node and nodes in $V'$. In such case an edge contraction operation can have linear complexity instead of quadratic in the number of nodes. Since we use a value of $k \in [1, 5]$ in all our algorithms utilizing incremental updates, they show such behaviour. This is also demonstrated in empirical analysis from Figure~\ref{fig:emp_comp_analysis}.  
\end{description}

\section{Influence of affinity strength}
On the Cityscapes dataset we compare panoptic quality on different object classes by varying the value of affinity strength $\alpha$ in 
~\eqref{eq:extended_features_similarity_ip}. The results are given in Table~\ref{tab:cityscapes_panoptic_quality_classes}. We observe that for classes contain many small objects large value of $\alpha$ is suitable whereas for classes with large objects small value of $\alpha$ is preferable. Although our default value of $0.4$ already makes dense multicut outperform the baseline, further improvement is still possible e.g. by tuning $\alpha$.
\begin{table}[h]
\caption{Comparison of panoptic segmentation on Cityscapes dataset for different values of affinity strength $\alpha$~\eqref{eq:extended_features_similarity_ip}. All results are computed using the \texttt{DAppLAEC} algorithm. Largest values in each row are highlighted with bold.}
\label{tab:cityscapes_panoptic_quality_classes}
\begin{center}
\begin{tabular}{l r r r r r r r r r}
\toprule
& \multicolumn{8}{c}{Panoptic quality on varying values of $\alpha$} \\
\cmidrule(lr){2-10}
Category & 0.1 & 0.2 & 0.3 & 0.4 & 0.5 & 0.6 & 0.7 & 0.8 & 0.9 \\
\midrule
\textit{Person} &  31.5 & 38.1 & 43.2 & 46.9 & 49.8 & 52.6 & 54.3 & \textbf{55.0} & 52.4 \\
\textit{Rider} & 51.1 & 53.0 & 53.9 & 54.5 & \textbf{55.5} & 55.4 & 53.9 & 51.0 & 45.5 \\
\textit{Car} & 45.6 & 52.9 & 57.8 & 60.5 & 63.3 & \textbf{64.8} & 64.1 & 62.2 & 57.8 \\
\textit{Truck} & \textbf{54.1} & 53.7 & 52.7 & 52.3 & 49.0 & 47.8 & 45.4 & 41.5 & 34.7 \\
\textit{Bus} & \textbf{75.1} & 74.2 & 73.5 & 71.2 & 69.3 & 63.6 & 58.5 & 54.5 & 47.3 \\
\textit{Train} & \textbf{75.0} & 74.9 & 71.5 & 62.9 & 56.3 & 51.7 & 45.1 & 40.4 & 32.3 \\
\textit{Motorcycle} & 45.5 & 46.1 & 48.0 & 46.8 & 48.7 & \textbf{49.1} & 47.8 & 45.2 & 39.8 \\
\textit{Bicycle} & 38.1 & 43.2 & 45.6 & 46.9 & 47.8 & \textbf{48.0} & 46.9 & 44.6 & 40.4 \\
\midrule
Average ($PQ_{th}$) & 52.0 & 54.5 & \textbf{55.8} & 55.2 & 55.0 & 54.1 & 52.0 & 49.3 & 43.8 \\
\bottomrule
\end{tabular}
\end{center}
\end{table}


\end{document}